\documentclass[12pt]{article}

\usepackage[utf8]{inputenc}
\usepackage{amsmath}
\usepackage{amsfonts}
\usepackage{mathrsfs}
\usepackage{amssymb}
\usepackage{amsthm}
\usepackage{bm}

\usepackage{libertine}
\usepackage{graphicx}
\usepackage{subfigure}
\usepackage{csquotes}

\newtheorem{theorem}{Theorem}
\newtheorem{corollary}{Corollary}
\newtheorem{lemma}{Lemma}

\newtheorem{proposition}{Proposition}

\usepackage[ruled, linesnumbered, boxed]{algorithm2e}

\usepackage{color}

\usepackage[bookmarks=false]{hyperref}
\hypersetup{hidelinks}

\title{Instance-Dependent PU Learning by \\ Bayesian Optimal Relabeling}

\font\myfont=cmr12 at 13pt

\author{\myfont Fengxiang~He\thanks{He, Liu, and Tao are with the UBTECH Sydney Artificial Intelligence Centre and the School of Computer Science, in the Faculty of Engineering, at The University of Sydney, 6 Cleveland St, Darlington, NSW 2008, Australia (email: {fengxiang.he, tongliang.liu, dacheng.tao}@sydney.edu.au).} \ \ \ \ \ Tongliang~Liu\footnotemark[1] \ \ \ \ \ Geoffrey~I~Webb\thanks{Webb is with Faculty of Information Technology, Monash University, Clayton, VIC 3800, Australia, (email: geoff.webb@monash.edu).} \ \ \ \ \ Dacheng Tao\footnotemark[1]}

\date{}

\begin{document}

\maketitle

\begin{abstract}
  When learning from positive and unlabelled data, it is a strong assumption that the positive observations are randomly sampled from the distribution of $X$ conditional on $Y = 1$, where X stands for the feature and Y the label. Most existing algorithms are optimally designed under the assumption. However, for many real-world applications, the observed positive examples are dependent on the conditional probability $P(Y = 1|X)$ and should be sampled biasedly. In this paper, we assume that a positive example with a higher $P(Y = 1|X)$ is more likely to be labelled and propose a probabilistic-gap based PU learning algorithms. Specifically, by treating the unlabelled data as noisy negative examples, we could automatically label a group positive and negative examples whose labels are identical to the ones assigned by a Bayesian optimal classifier with a consistency guarantee. The relabelled examples have a biased domain, which is remedied by the kernel mean matching technique. The proposed algorithm is model-free and thus do not have any parameters to tune. Experimental results demonstrate that our method works well on both generated and real-world datasets.
\end{abstract}

\newpage

\section{Introduction}

Instances are required to be labelled as either positive or negative in traditional binary classification tasks. However, a new data setting emerges out in recent decades and breaks this convention. This is because for many situations, only positive labels are identifiable while negative ones are not. For example, a number of datasets in molecular biology \cite{galperin2007molecular, elkan2008learning} contain proteins that are known to have particular functions. Under some restrictions, instances can only be labelled when the functions are active and observed. This does not mean that the unlabelled instances do not have the functions. Another example is to recognize customers who are interested in one product from the customer profiles. The customers who have bought the product can be considered as positive examples. However, the others still have varying potentials to be interested. Those situations are normal in many areas where we can only label a group of positive examples and leave all others without labels, such as in biology \cite{yang2014ensemble}, online commerce \cite{ren2014positive}, and cyber security \cite{zhang2017poster}. Thus, we are facing a great amount of positive and unlabelled data (PU data) in real-world applications. An important problem is raised that how to efficiently learn a classifier from positive and unlabelled data (PU learning).

A basic question for PU learning is how to describe the phenomenon that some positive examples are labelled, but the others are not. To address this issue, we use a {\it one-side noise} framework used in many previous works; see, for example, \cite{elkan2008learning, letouzey2000learning, natarajan2013learning}. Specifically, we treat the unlabelled examples as negative ones. By this way, the PU datasets become full-labelled datasets in which negative instances all have true labels while positive ones may have wrong labels. Under the one-side noise framework, we could model the miss-label phenomenon by {\it mislabelled rates} $\rho_+(X)$ and $\rho_-(X)$ which respectively express the probability that a positive or negative instance is unlabelled:
\begin{gather}
	\rho_+(X) = P(\tilde Y = +1 | Y = -1, X), \nonumber\\
	\rho_-(X) = P(\tilde Y = -1 | Y = +1, X). \nonumber
\end{gather}where $X$ is the instance, $Y \in \{-1, +1\}$ is the true label which is treated as a latent variable due to its unavailability, and $\tilde Y \in \{+1, -1\}$ is the corresponding observed label. Generally, we define the following notation for all examples including both positive and negative ones:
\begin{equation}
    \rho(X, Y) = \begin{cases}
    	\rho_+(X) &, Y = +1, \\
        \rho_-(X) &, Y = -1,
    \end{cases}
\end{equation}
Straightly from the fact that the observed labels $\tilde{Y}$ of  all negative instances ($Y = -1$) are correct, we can get $\rho_-(X) = 0$.

Many existing methods assume that the mislabelled rate is a constant for all positive examples; see, for example, \cite{elkan2008learning}. A more general and realistic PU learning model would have its mislabelled rate being instance-dependent, which assumes that the mislabelled rate depends on the instance. Therefore, we propose a new hypothesis for the basic question: whether an example is labelled or not depends on its feature. Furthermore, the more difficult to label an instance, the higher its mislabelled rate $\rho(X, Y)$ is.

To mathematically measure the difficulty of labelling an observation, we defines the {\it probabilistic gap} that equals to the difference between the Bayesian posteriors $P(Y=1|X)$ and $P(Y=-1|X)$ given data, i.e.,
\begin{equation}
	\Delta P(X) = P(Y=1 | X) - P(Y = -1 | X).
\end{equation}
Intuitively, the probabilistic gap expresses the distance from a positive instance to the classifier. The smaller the probabilistic gap, the more difficult to label the instance. Therefore, we propose to model the mislabelled rate $\rho(X, Y)$ as a monotone decreasing function of its corresponding probabilistic gap $\Delta P(X)$. In this paper, we name this model as {\it probabilistic-gap PU model} (PGPU).

However, $P(Y | X)$ is not accessible due to the unavailability of the latent true label $Y$. Therefore, we cannot get $\Delta P(X)$ directly. To handle this problem, we develop a method to estimate $\Delta P(X)$. Let
\begin{equation}
	\Delta \tilde{P}(X) = P(\tilde Y = 1 | X) - P(\tilde Y = -1 | X).
\end{equation}
In contrast to $\Delta P(X)$, $\Delta \tilde{P}(X)$ can be estimated by many well-developed probability density estimation methods. Based on the PGPU model, we could further construct a mapping from $\Delta \tilde{P}(X)$ to $\Delta P(X)$ by exploiting $\rho(X, Y)$. With the information of $\Delta P(X)$, a {\it Bayesian optimal relabelling} method is then designed to re-label a group of examples. The new labels are identical to the ones assigned by the Bayesian optimal classifier (details will be given in Section 5). By this method, we can automatically and largely extend the labelled examples with theoretical guarantees. The extended labelled dataset contains both negative and positive examples. A classifier is then learned on this dataset. It should be noted that the domain of the new labelled dataset is biased, because the observations in a sub-domain of the feature domain could never be labelled. To remedy this bias, we use a reweighting technique, kernel mean matching \cite{huang2007correcting}, when training the classifier. Our method is then evaluated on both generated and real-world datasets. Empirical results prove our method's feasibility.

The rest of this paper is structured as follows. In Section 2, the related work is introduced. In Section 3, we formalize our research problem. In Section 4, 5, and 6, we respectively present the PGPU model, the Bayesian optimal relabelling method, and the reweighting classification method in detail. In Section 7, a theoretical analysis is provided. In Section 8, we provide the experimental results. In Section 9, we conclude our paper and discuss the future work.

\section{Related Work}\label{sec:relatedwork}

Many practical methods dealing with PU learning have been delivered in recent decades. Their frameworks can be roughly divided into two classes: one is with the main focus being on exploiting how to extend the labelled datasets; the other handles the challenge by constructing a probabilistic model and learning classifiers on the positive and unlabelled data. It should be noted that some articles in the second part are motivated by the approaches for classification tasks with label noise \cite{kearns1993learning, cesa1999sample, bshouty2002pac}., and sometimes are called {\it one-side label noise model}\footnote{The unlabelled examples are treated as noisy negative ones. Therefore, the label noises in the whole dataset is only because of the false negative examples. Methods dealing with this setting is called the one-side label noise model.}.

To the first end, a major thrust has been on developing methods that extract reliable negative (or positive, equivalently) from the unlabelled data, e.g., through iterative 2-step processes. In the first step, standard classification methods, such as SVM \cite{li2003learning, yu2002pebl} and logistic regression \cite{lee2003learning}, are used to identify negative data points with significant confidence from unlabelled examples by the classifier learned in the given positive dataset. In the next step, the extracted negative examples are used with the original positive examples to refine the previous classifier. These two steps run iteratively to select confident labelled examples from unlabelled datasets. Similar to these iterative classification methods, EM-like algorithms are also proposed for the PU learning \cite{liu2003building, liu2002partially}. One important disadvantage of these methods is that the extraction of the reliable labels lacks theoretical guarantees.

The second way to learn from the PU data is to treat the unlabelled examples as negative ones; therefore, the PU datasets are transferred into full-labelled datasets with label noises. Elkan and Noto assume that all observed positive examples are randomly drawn from the positive examples, and then use a reweighting technique to adapt traditional binary classification methods to be robust in learning with PU data \cite{elkan2008learning}. Most models here are based on a strong assumption that the mislabelled rate is a constant. In this paper, we have a weaker assumption: the positive example closer to the latent optimal classifier is more difficult to be labelled. In the rest of this paper, we will introduce why it is feasible. Some other work avoids any assumption but introduces more labelling labour; see, for example, Letouzey et al. use positive statistical queries to estimate the distribution of positive instance space and instance statistical queries to estimate the distribution of the whole instance space. Based those estimations, they further develop a binary classification algorithm\cite{letouzey2000learning}.

Furthermore, as the one-side label noise model suggests, PU data is a degenerate case of the label noise problem, where both positive and unlabelled examples can be mislabelled. However, for PU learning, the positive examples are always accurately labelled. Therefore, the PU data can also be solved directly by class-dependent label noise methods. For example, Natarajan et al. assume that the noise rate depends on the class and provides an unbiased estimator for the noise rate \cite{natarajan2013learning}. Based on the estimator, they further propose a reweighting method to learn a classifier from the noisy data. Liu and Tao also assume the noise rate is a constant for each class, and find a data-dependent upper bound for the noise rate, which leads to an estimation method \cite{liu2016classification}. A reweighting classification method is then employed to learn a classifier from the noisy data. Label noise methods are also applicable in many similar areas, such as transfer learning \cite{yu2017transfer}, and learning classifier from data with complementary labels \cite{ishida2017learning, yu2017learning}.

The practical methods of PU learning also yield many concerns about the consistency and the class-prior estimation. As the unlabelled data contains both positive and negative examples, simply separating the positive and unlabelled data could lead to biased solutions even when reweighting techniques are employed. Under a restriction that the mislabelled rate is a constant, the consistency issue can be completely solved by using the loss functions $l(z)$ such that $l(z)+l(-z) = \text{constant}$. A common choice for such loss functions is the ramp loss which leads to the robust support vector machine \cite{collobert2006trading, wu2007robust}. However, the ramp loss is non-convex which could be computationally expensive. To address this issue, Plessis et al. propose a convex framework, where the learned classifiers are proved to be asymptotically consistent to the optimal solution with a lower computational cost \cite{du2015convex}. In addition, many existing methods rely on the estimation of the class priors. However, the absence of observed negative examples can also bring biases into the estimation. \cite{scott2009novelty,blanchard2010semi} prove that the bias of the class-prior estimation can be remedied by the Neyman-Pearson classification. Meanwhile, Liu and Tao proves that the class priors are bounded by the conditional probabilities $P(\tilde Y | X)$ and then provide an estimator as the bound can be easily reached \cite{liu2016classification}. Also, Plessis et al. hire the penalized $L_1$-distance for the estimation to avoid the bias \cite{du2015class}. Additionally, the class-prior estimation problem can be solved under the framework of {\it mixture proportion estimation} (MPE) \cite{woodward1984comparison,scott2015rate,redner1984mixture}. By setting a restriction weaker than other works that the component distributions are independent from each other, Yu et al. propose an MPE method with the state-of-art performance \cite{yu2018efficient}.

PU learning methods can also extended to the multi-class classification problem; see, for example, Xu et al. propose an algorithm with strong theoretical guarantees that the generalization bound is no worse than $k \sqrt{k}$ times of the fully-supervised multi-class classification methods \cite{xu2017multi}. In the recent decade, PU learning methods have been applied to many areas, such as web mining \cite{yu2004pebl, chou2018learning}, and health data mining\cite{chen2016mining}. Additionally, PU learning is related with semi-supervised learning closely, which deals with the datasets that contains many unlabelled instances. Therefore, we could directly use supervised learning methods to deal with PU data; see, for example, \cite{luo2018semi,akbarnejad2018efficient,yu2018semi}.

\section{Problem Setup}\label{sec:problemsetup}

Suppose $S = \{(x_1, s_1), \ldots, (x_n, s_n)\}$ is a positive and unlabelled dataset, where $x_i \in \mathbb{R}^d$ is the instance and $s_i \in \{1, \text{NULL}\}$ is the corresponding observed label. Here, the label $1$ means positive and the label $\text{NULL}$ means unlabelled. In the same time, each example has a latent true label which is not accessible. We denote the true labels as $y_i \in \{ \pm 1 \}, i = 1, \ldots, n$. For simplicity, we rewrite the $\text{NULL}$ as the negative ($-1$). Therefore, in this new dataset, all the positive labels are clean while the negative labels are corrupted. In other words, we will solve the PU learning problem by employing the one-side label noise model. Before this, we introduce some preliminary results.

Bayesian optimal classifier assigns the labels with the highest posterior probabilities to instances \cite{duda1973pattern, bousquet2004introduction}. In binary classification circumstances, the Bayesian optimal classifier can be written as
\begin{equation}
\label{eq:bayes}
\hat Y(x) = \begin{cases}
+1,& P_+ - P_- > 0, \\
\text{randomly selection},& P_+ - P_- = 0, \\
-1,& P_+ - P_- < 0.
\end{cases}
\end{equation}
where $P_+ = P(Y = +1 | X)$ and $P_- = P(Y = -1 | X)$. This result can also be written as the following lemma. We will need it later.

\begin{lemma}
	For any distribution $D$, suppose the examples satisfy $(X, Y) \sim D$. Then the classifier satisfies
    \begin{equation}
    	f(X) = \text{sign}(P(Y = 1 | X) - 1/2),
    \end{equation}
    is a Bayes optimal classifier. 
\end{lemma}

When learning classifiers from the data, a criterion to measure the performance of the classifier is necessary. Here, we define the {\it risk} to play the role. Suppose all potential classifiers $h$ constitute a hypothesis space $H$. Suppose the unknown distribution of data is $D$. Additionally, we use a loss function $l$ to measure the accuracy for a prediction. Then the risk of the classifier $h$ is defined as
\begin{equation}
R(h, D, l) = E_{(X, Y) \sim D} [l(h(X), Y)],
\end{equation}
The classifier with the highest performance is defined as
\begin{equation}
h^* = \arg \min_{h \in H} R(h, D, l).
\end{equation}

However, as the distribution of the data is unknown, we cannot compute the risk directly from the data. Therefore, we define the {\it empirical risk} on the dataset $S$ to estimate the risk $R(h, D, l)$. The empirical risk is defined as
\begin{equation}
\hat R(h, S, l) = \frac{1}{n} \sum_{i = 1}^n [l(f(x_i), y_i)].
\end{equation}
We can thus estimate $h^*$ by the minimization of the empirical risk:
\begin{equation}
\hat h = \arg \min_{h \in H} \hat R(h, S, l).
\end{equation}
This process is named as the {\it empirical risk minimization} (ERM) \cite{vapnik2013nature}. Bartlett et al. prove that the consistency of $R(\hat h, S, l)$ to $R(h^*, D, l)$ is guaranteed, when the $l$ is a surrogate loss function \cite{bartlett2006convexity}.

\section{Probabilistic-Gap PU Model}

In PU learning, the observed label $\tilde Y$ could be different from the latent true label $Y$. 
Specifically, we have following formulations:
\begin{gather}
	P(\tilde Y = -1 | X, Y = -1) = 1, \\
    P(\tilde Y = +1 | X, Y = -1) = 0, \\
    P(\tilde Y = -1 | X, Y = +1) > 0, \\
    P(\tilde Y = +1 | X, Y = +1) > 0.
\end{gather}
We deal with this phenomenon as a probabilistic problem and use the {\it mislabelled rate} $\rho(X, Y)$ to express it (it is defined by Equation (1)).

In this paper, we define a {\it probabilistic gap} to express the difficulty of labelling a positive instance and then propose a {\it probabilistic-gap PU} (PGPU) model for the mislabelled rate based on the probabilistic gap. The probabilistic gap is defined as the following formulation:
\begin{equation}
	\Delta P(X) = P(Y = +1 | X) - P(Y = -1 | X).
\end{equation}
An important property of the probabilistic gap is that the hyperplane expressed by the equation $\Delta P(X) = 0$ is the Bayesian optimal classifier of the PU dataset. This property is straight-forward from the definition of the Bayesian optimal classifier (see eq. \ref{eq:bayes}).

Intuitively, the smaller the probabilistic gap, the more difficult to label the instance. Therefore, we propose to model the mislabelled rate $\rho(X, Y)$ as a monotone decreasing function of its corresponding probabilistic gap $\Delta P(X)$. We assume there exists a function $f$ such that
\begin{equation}
\rho(X, Y) = f(\Delta P(X)).
\end{equation}
Furthermore, we also assume that the function $f$ has the following properties:
\begin{enumerate}
\item The mislabelled rate $\rho(X, Y)$ is higher than $0$ when the probabilistic gap $\Delta P(X) < 0$, while it vanishes where the probabilistic gap $\Delta P(X) < 0$, i.e.,
\begin{gather}
\label{eq:prop_1}
	f(x)|_{x>0} = \rho_+ > 0,\\
\label{eq:prop_2}
    f(x)|_{x<0} = \rho_- = 0.
\end{gather}

\item The mislabelled rate $\rho(X)$ should be monotone decreasing in the term of the probabilistic gap $\Delta P(X)$, i.e.,
\begin{equation}
\label{eq:prop_3}
    f'(x) |_{x>0} = \rho_+'(X) < 0.
\end{equation}
\end{enumerate}

These properties are respectively due to the following two reasons:
\begin{enumerate}
\item All negative instances in PU datasets have correct labels, while some positive ones are incorrectly labelled as negative. Therefore, the observed positive examples are free from corruptions, and the negative ones have label noise. Additionally, $\Delta P(X) = 0$ is a threshold for Bayes optimal classifier. Therefore, we use $\Delta(X) > 0$ and $\Delta(X) < 0$ to respectively express the positive and negative examples.
\item It would be reasonable that there exists a latent optimal boundary that divides the feature domain into two sub-domains respectively for positive and negative, although we don't know where the boundary is. The positive instances closer to the boundary are more difficult to label compared with the ones further away. Therefore, the examples closer to the classifier have higher mislabelled rates. Thus, it would be reasonable to assume that the mislabelled rate should increase while the probabilistic gap $\Delta P(X)$ decreases, as $\Delta P(X)$ can be used to express the distance between the instance and the optimal boundary.
\end{enumerate}

Our method is elegant because we do not assume any exact formulation of the function $f$. We will discuss this later.

\section{Bayesian Optimal Relabeling}\label{bayesianoptimalrelabeling}

In this section, we propose a {\it Bayesian optimal relabelling} method to label a group of unlabelled examples. These new labels are assigned according to the probabilistic gap $\Delta P(X)$. We will prove that they are coincident with the ones that the Bayesian optimal classifier would assign.

\subsection{Observed Probabilistic Gap}

The probabilistic gap $\Delta P(X)$ cannot be computed directly from the data, as the true labels are not accessible. Therefore, we define the {\it observed probabilistic gap} $\Delta \tilde{P}(X)$ in order to infer the probabilistic gap $\Delta P(X)$. Then, we can label a group of unlabelled examples. Similar with the probabilistic gap $\Delta P(X)$, the observed probabilistic gap $\Delta \tilde{P}(X)$ is defined as in the following formulation:
\begin{equation}
	\Delta \tilde P(X) = P(\tilde Y = +1 | X) - P(\tilde Y = -1 | X).
\end{equation}

Applying Bayes theorem, we can divide the posterior probability $P(\tilde Y = +1 | X)$ as the following formulations
\begin{align}
& P(\tilde Y = +1 | X) \nonumber \\
= & P(\tilde Y = +1, Y = +1 | X) + P(\tilde Y = +1, Y = -1 | X) \nonumber \\
= & P(\tilde Y = +1 | Y = +1, X) P(Y = +1 | X) \nonumber \\
& + P(\tilde Y = +1 | Y = -1, X) P(Y = -1 | X) \nonumber \\
= & (1 - \rho_+(X)) P(Y = +1 | X).
\end{align}
Similarly, we can divide the posterior probability $P(\tilde Y = -1 | X)$ as
\begin{align}
& P(\tilde Y = -1 | X) \nonumber \\
= & P(\tilde Y = -1, Y = +1 | X) + P(\tilde Y = -1, Y = -1 | X) \nonumber \\
= & P(\tilde Y = -1 | Y = +1, X) P(Y = +1 | X) \nonumber \\
& + P(\tilde Y = -1 | Y = -1, X) P(Y = -1 | X) \nonumber \\
= & \rho_+(X) P(Y = +1 | X) + P(Y = -1 | X) \nonumber \\
= & \rho_+(X) + (1 - \rho_+(X)) P(Y = -1 | X).
\end{align}
Therefore, we finally get the relation between $\Delta \tilde{P}(X)$ and $\Delta P(X)$
\begin{align}
\label{eq:observe_true_probs}
& \Delta \tilde{P} (X) \nonumber \\
= & P(\tilde Y = +1 | X) - P(\tilde Y = -1 | X) \nonumber \\
= & (1 - \rho_+(X)) P(Y = +1 | X) - \rho_+(X) P(Y = +1 | X) \nonumber \\
& - P(Y = -1 | X) \nonumber \\
= & (1 - 2 \rho_+(X)) P(Y = +1 | X) - P(Y = -1 | X) \nonumber \\
= & (2 - 2 \rho_+(X)) P(Y = +1 | X) - 1 \nonumber \\
= & (2 - 2 \rho_+(X)) \frac{2P(Y = +1 | X)}{2} - 1 \nonumber \\
= & (2 - 2 \rho_+(X)) \frac{P(Y = +1 | X) + 1 - P(Y = -1 | X)}{2} - 1 \nonumber \\
= & (2 - 2 \rho_+(X)) \frac{\Delta P(X) + 1}{2} - 1 \nonumber \\
= & (1 - \rho_+(X)) (\Delta P(X) + 1) - 1.
\end{align}


\subsection{A Bayesian Optimal Relabelling Method}

From eq. (\ref{eq:observe_true_probs}), a theorem can be summarized as follows.\footnote{This theorem is given by \cite{bekker2018learning} based on the results in a previous version of this paper.}

\begin{theorem}
	For PU data, we have the following two properties:
	\begin{itemize}
	\item
	Observed probabilistic gap $\Delta \tilde P(X)$ is not larger than probability gap $\Delta P(X)$; i.e., for any $X$, we have that
	\begin{equation}
		\Delta \tilde P(X) \le \tilde P(X).
	\end{equation}
	
	\item
	Observed probabilistic gap $\Delta \tilde P(X)$ has the same ordering of probabilistic gap $\Delta P(X)$; i.e., for any $X_1$ and $X_2$, we have that
	\begin{gather}
	\label{eq:equal_gap}
	\Delta \tilde P(X_1) = \Delta \tilde P(X_2) \Leftrightarrow \Delta P(X_1) = \Delta P(X_2);\\
	\Delta \tilde P(X_1) \le \Delta \tilde P(X_2) \Leftrightarrow \Delta P(X_1) \le \Delta P(X_2).
	\end{gather}
	\end{itemize}
\end{theorem}

The proof is given by \cite{bekker2018learning}. We recall it here with slight modifications to make this paper completed.

\begin{proof}
From eq. (\ref{eq:observe_true_probs}),
\begin{align}
\Delta \tilde{P} (X) = & (1 - \rho_+(X)) (\Delta P(X) + 1) - 1 \nonumber\\
= & (1 - \rho_+(X)) \Delta P(X) + 1 - \rho_+(X) - 1 \nonumber\\
= & (1 - \rho_+(X)) \Delta P(X) - \rho_+(X).
\end{align}
Because
\begin{equation}
\rho_+(X) = P(\tilde Y = -1 | X, Y = +1) \in (0, 1),
\end{equation}
we have that
\begin{gather}
1 - \rho_+(X) < 1.
\end{gather}
Therefore,
\begin{align}
\Delta \tilde{P} (X) = (1 - \rho_+(X)) \Delta P(X) - \rho_+(X) < \Delta P(X).
\end{align}

Additionally, because $\Delta \tilde P(X)$ monotonically decreases as $\Delta P(X)$ increases (see eq. \ref{eq:observe_true_probs}), we can directly have Property 2.
\end{proof}

Based on this theorem, we can obtain a relabeling method.

\begin{theorem}
	Define that
\begin{equation}
l = - \rho_+ (X) |_{\Delta P(X) = 0} \in (-1, 0).
\end{equation}
The unlabelled instances in the area $\Delta \tilde{P}(X) \in [-1, l]$ can be labelled as negative points. Meanwhile, The unlabelled instances in the area $\Delta \tilde{P}(X) \in (l, 1]$ can be labelled as positive. All new labels are identical to the ones that the Bayesian optimal classifier assigns.
\end{theorem}

\begin{proof}
When $\Delta P(X) = 0$, $\Delta \tilde P(X) = l$. Therefore, when $\Delta \tilde P(X) > l$, $\Delta P(X) > 0$. Otherwise, $\Delta P(X) < 0$.
\end{proof}

However, due to the restrictions of sampling methods to approximate probabilities, we propose a weaker version as follows. Both versions are rigorously correct under the assumptions of function $f$ (see eqs. \ref{eq:prop_1}, \ref{eq:prop_2}, and \ref{eq:prop_3}).

\begin{corollary}
\label{cor:optimal_relabel}
	The unlabelled instances in the area $\Delta \tilde{P}(X) \in [-1, l]$ can be labelled as negative points. Meanwhile, The unlabelled instances in the area $\Delta \tilde{P}(X) \in (0, 1]$ can be labelled as positive. All new labels are identical to the ones that the Bayesian optimal classifier assigns.
\end{corollary}

Based on Corollary \ref{cor:optimal_relabel}, we can formally state a {\it Bayesian optimal relabelling method} to label a group of unlabelled examples. Firstly, we employ an SVM-based method to estimate the conditional probability $P(\tilde Y|X)$. Platt et al. develop this technique based on the standard SVM to provide the calibrated posterior probabilities of the examples corresponding to the potential classes \cite{platt1999probabilistic}. This work hires a sigmoid function to map the outputs of the standard SVM to the posterior probabilities. Existing main SVM tool libraries usually have included this technique. In this paper, we use the LIBSVM toolkit \cite{chang2011libsvm} to estimate the probabilities mentioned above. Secondly, we calculate the empirical probabilistic gap by definition. Afterwards, the boundary $l$ can be estimated by $2$ estimators. The first one is as the following equation
\begin{equation}
	\hat l_1 = \inf_{y = +1} \Delta \tilde{P}(x).
\end{equation}
This estimator is straight-forward from the Theorem 1. However, this estimator relies on the minimization of the $\Delta \tilde{P}(x)$, which can be vulnerable to the potential outliers. To deal with this problem, we calculate the mean of the $n'$ smallest $\Delta \tilde{P}(x)$ to estimate the $l$ in practice. The second way to estimate the boundary $l$ is by cross-validation. This estimation method could be more robust. We search the optimal value of the boundary $l$ in a potential area according to the best performance of the whole classification algorithm during the validation. Finally, we can label the instances in the areas $\Delta P(X) \in (-1, l)$ and $\Delta P(X) \in (0, 1)$ respectively as positive and negative. The new full-labelled dataset is denoted as $\hat S$. Here, we suppose the dataset $\hat S = \{ (x'_i, \tilde y'_i), i = 1, \ldots, N \}$ satisfies a distribution $\hat D$. We will learn a classifier on this dataset.

\section{Learn a Classifier from PU Data}

Standard classification algorithms minimize empirical risk $\hat{R}(h, S, l)$ on training data $S$ trying to minimize the expected risk $R(h, D, l)$ (a small expected risk implies a small test error). This method is based on the fact that the expectation of the empirical risk equals to the expected risk, i.e.,
\begin{align}
& E_{S \sim D^{|S|}} \hat{R}(h, S, l) = E_{S \sim D^{|S|}} \frac{1}{|S|} \sum_{(x, y) \in S} l(h(x), y) \nonumber \\
= & E_{(X, Y) \in D} l(h(X), Y) = R(h, D, l),
\end{align}
where $l$ is the loss function, and $D$ is the distribution of $(X, Y)$. Therefore, optimizing the empirical risk is actually optimizing the expected risk when the sample size is large enough.

The result mentioned above is based on an assumption that the probabilistic distributions of the training set and the test set are coincident, which however does not hold in our case -- even the supports are not the same. This is because that the unlabelled examples in a sub-domain of the data, which satisfy $\Delta \tilde{P} \in [l, 0]$, can never be labelled by our proposed algorithm.

To handle this issue, Huang et al. propose an importance reweighting technique to modify the standard empirical risk to be asymptotically consistent to the risk \cite{huang2007correcting}. The following development is also mentioned by Cheng et al. \cite{cheng2017learning}.

Here, we make an important assumption that the conditional probability $P(Y|X)$ remains the same while the probability $P(X)$ differs. In other words, the bias in data domain may not affect the mapping from the instance to the label. We thus have the following equation
\begin{align}
R(h, D, l) & = E_{(X, Y) \sim D} l(h(X), Y) \nonumber \\
& = E_{(X, Y) \sim \hat D} \frac{P_{D}(x, y)}{P_{\hat D}(x, y)} l(h(X), Y) \nonumber \\
& = R(h, \hat D, \frac{P_{D}(x, y)}{P_{\hat D}(x, y)} l) \nonumber \\
& =  E_{\hat S \sim \hat D^{|\hat S|}} \hat{R}(h, \hat S, \frac{P_{D}(x, y)}{P_{\hat D}(x, y)} l) \nonumber \\
& =  E_{\hat S \sim \hat D^{|\hat S|}} \hat{R}(h, \hat S, \frac{P_{D}(x)}{P_{\hat D}(x)} l) \nonumber \\
& \triangleq E_{\hat S \sim \hat D^{|\hat S|}} \hat{R}(h, \hat S, \beta(x) l),
\end{align}
where $\beta(x) = \frac{P_{D}(x)}{P_{\hat D}(x)}$.

However, the $P_D(x)$ is not accessible. Thus, obtaining the $\beta(x)$ could be an issue. Here, we employ the kernel mean matching (KMM) technique to obtain the $\beta(x)$ via an optimization process. For simplicity, we write the selected data as $\{X_{\text{selected}}\}$ in the sample domain $\mathcal{X}$.

Suppose there is a feature mapping $\Phi: \mathcal{X} \to \mathcal{H}$, where $\mathcal{H}$ is a reproducing kernel Hilbert space (RKHS) introduced by a universal kernel $k(X, X') = <\Phi(X), \Phi(X')>$. Applying Theorem 1 in \cite{huang2007correcting}, we can get the following result
\begin{gather}
E_{X \sim D} \Phi(X) = E_{X \sim \hat{D}} \beta(X) \Phi(X) \nonumber \\
\Leftrightarrow P_D(X) = \beta(X) P_{\hat D}.
\end{gather}
Therefore, we can obtain the $\beta(X)$ by solving the following optimization problem:
\begin{gather}
	\min_{\beta(X)} \| E_{X \sim P_D(X)} \Phi(X) -  E_{X \sim \hat{D}} \beta(X) \Phi(X) \|, \nonumber \\
    \text{s.t. } \beta(X) > 0 \text{ and }  E_{X \sim \hat{D}} \beta(X) = 1.
\end{gather}
However, both $P_D(x)$ and $P_{\hat D}(x)$ are not available. Therefore, we use their empirical counterparts to estimate them. Then we get the final optimization problem.
\begin{gather}
	\min_{\beta} \|\frac{1}{n} \sum_{i = 1}^n \Phi(x_i) -  \frac{1}{n'} \sum_{i = 1}^{n'} \beta_i \Phi(x'_i) \|, \nonumber \\
    \text{s.t. } \beta_i \in [0, \beta] \text{ and }  | \frac{a}{n'} \sum_{i = 1}^{n'} \beta_i - 1 | \le \epsilon,
\end{gather}
where $\beta = (\beta_1, \ldots, \beta_{n'})$ and $\epsilon$ is a small real constant.

Now, we can finally summarize our algorithm in Algorithm 1.

\begin{algorithm}
    \normalsize
	\SetAlgoLined
	\KwIn{PU sample $\tilde{D} = \{(x_1, s_1), \ldots, (x_n, s_n)\}$}
	\KwOut{Classifier $\hat{h}$}
	\textbf{Step 1. conditional Bayesian optimal relabelling}\;
	Estimate probabilities $\tilde{P}_+(x)$ and $\tilde{P}_-(x)$ of instances assigned in particular classes\;
	$\Delta \tilde{P}(x) = \tilde{P}_+(x) - \tilde{P}_-(x)$\;
	Find the lower threshold $\Delta \tilde{P}(x)_L = inf_{y = +1} \Delta \tilde{P}(x)_L$\;
	Select instances satisfy $\Delta \tilde{P}(x) > 0$ or $\Delta \tilde{P}(x) < \Delta \tilde{P}(x)_L$\;
	\textbf{Step 2. Learn a classifier}\;
	Estimate $P_D(X)$ and $P_{D^*}(X)$\;
	$\beta(x) = \frac{P_{D}(x)}{P_{D^*}(x)}$\;
	Run weighted SVM on selected data with weights $\beta$\;
	\Return{The classifier learned by the weighted SVM.}
	\caption{Probabilistic-gap PU Classification algorithm}
\end{algorithm}

\section{Theoretical Analysis}

In this section, we theoretically analyze our proposed method. In the beginning, we analyze our method under two restrictions: (1) the examples relabelled by the Bayesian optimal relabelling process lie in the same support with the original PU data; and (2) the supports of the latent clean positive and negative data are non-overlapped. Furthermore, we discuss the general circumstances without any restriction.


The main results of our method with the two restrictions are Theorem 3 and 4.

\begin{theorem}
	Assume the supports of the latent clean positive and negative data are non-overlapped. Suppose the original PU data satisfies a distribution $D$ while the data collected by the proposed Bayesian optimal relabelling process satisfies another distribution $\hat D$. Assume the distributions $D$ and $\hat{D}$ have the same support. Then, the Bayesian optimal classifier $h^*_{D}$ on the distribution $D$ coincides with $h^*_{\hat{D}}$ on the distribution $\hat{D}$.
\end{theorem}

\begin{proof}
	Any instance $X_i$ in the support $\text{supp}(P_D(X))$ of the distribution $D$ also lies in the support of the distribution $\hat{D}$, i.e., $X_i \in \text{supp}(P_{\hat{D}}(X))$. Additionally, applying Lemma 1, the theoretical optimal classifier $h^*_{\hat{D}}$ on the distribution $\hat{D}$ satisfies that
    \begin{align}
    	h^*_{\hat{D}}(X_i) = & \text{sign} (P_{\hat{D}}(Y = +1 | X_i) - 1/2) \nonumber \\
        = & \text{sign} (\mathbb{I}[h^*_D(X_i) = +1] - 1/2) \nonumber \\
        = & h^*_D(X_i).
    \end{align}
The second equation is guaranteed by the assumption that the supports of the latent clean positive and negative data are non-overlapped.

    The proof is completed.
\end{proof}

Based on Theorem 3, we can obtain the following theorem.

\begin{theorem}
	Suppose $l$ is a classification-calibrated loss function while the conditions in Theorem 3 still hold. Assume that $l \in [0, b]$. Then for any $\delta \in (0, 1)$, with a probability at least $1 - \delta$, we have
    \begin{align}
    	& |R(\hat{h}_{\hat{D}}, \hat{D}, l) - R(h^*_{\hat{D}}, \hat{D}, l)| \nonumber \\
        \le & 2 \mathfrak{R}(l \circ \mathcal{H}) + 2b \sqrt{\frac{\log(1/\delta)}{2n}},
    \end{align}
	Here, $\mathfrak{R}$ is the Rademacher complexity which is defined as
	\begin{equation}
		\mathfrak{R}(H) = \mathbb E_{\sigma, X, \hat{Y}} \left[\sup_{h \in \mathcal{H}}\frac{2}{n}\sum_{i=1}^{n} \sigma_i l(h(X_i), \tilde Y_i)\right],
	\end{equation}
	where $\sigma = (\sigma_1, \ldots, \sigma_n)$ and $\sigma_i, i = 1, \ldots, n$ are Rademacher variables i.i.d. drawn from a symmetric distribution on $\{-1, 1\}$. Additionally,
	\begin{equation}
		l \circ H = \{l \circ h | h \in H \}.
	\end{equation}
\end{theorem}

To obtain Theorem 4, we need a theorem in \cite{mohri2012foundations} and a lemma proved by Anthony and Bartlett \cite{anthony2009neural}. 

\begin{lemma}(Theorem 3.2 in \cite{mohri2012foundations})
	Let $H$ be a class of classifiers $h$. Suppose that the data satisfies a distribution $D$, the training sample is $S$, and the training sample size is $m$. Then for any $h \in H$, we have the following formulation for the expected risk $R(h, D, l)$ and the empirical risk $\hat{R}(h, S, l)$ under a loss function $l$
    \begin{equation}
    	R(h,D,l) \le \hat{R}(h, S, l) + \mathfrak{R}(l \circ H) + \sqrt{\frac{\log \frac{1}{\delta}}{2m}}.
    \end{equation}
\end{lemma}

Lemma 2 provides a generalization bound via Rademarcher complexity.
    
The lemma by Anthony and Bartlett is as follows. 
\begin{lemma}
Suppose the loss function $l$ is classification-calibrated. Denote the learned classifier as $\hat h$. Then we have
  \begin{align}
  	& |R(h^*, D, l) - R(\hat h, D, l)| \nonumber \\
    \le & 2 \sup_{h \in H} |R(h, D, l) - \hat R(h, S, l)|.
  \end{align}
\end{lemma}

Lemma 3 demonstrates that if the excess risk $|R(h, D, l) - \hat R(h, S, l)|$ is small enough, which means that the algorithm generalizes well, the expected risk of the classifier $\hat h$ is consistent to the one of the theoretical optimal classifier $h^*$. Applying Lemma 2 and 3, Theorem 4 is straight-forward.

Theorem 3 and 4 guarantee the feasibility of our method under two restrictions mentioned above. However, these two restrictions do not always hold in general circumstances: (1) the Bayesian optimal realbelling process leads to the domain bias that the support of the relabelled data is different from the one of the original PU data; and (2) the supports of the latent clean positive and negative data are probably overlapped. Our method hires a reweighting technique, KMM, to address these issue. In the rest of this section, we aim to analyze our method without any restriction. The analysis mainly focuses on the Bayesian optimal relabelling process and the reweighting technique. The main result is Theorem 5.

\begin{theorem}
	Suppose the conditional probability of observed noisy labels given features is $P(\tilde Y |X)$ and the relabelled data size is $N$. Assume $\beta(X, \tilde{Y})l(f(X), \tilde Y)$ is controlled by an upper bound $b$. Suppose $h^*_{D, l}$ is the theoretical optimal classifier under $D$ with loss $l$, while $\hat h_{\hat{D}, \beta l}$ is the classifier learned by ERM under $\hat{D}$ with loss $\beta l$. Then, for any real-valued constant $\delta \in (0, 1)$, with probability at least $1 - \delta$, we have
	\begin{align}
		& \left|R(\hat h_{\hat{D}, \beta l}, \hat{D}, \beta l) - R(h^*_{D, l}, D, l) \right| \nonumber \\
        \le & 2 \max_{(X, \hat{Y})} \beta(X, \tilde Y) \mathfrak{R}(l \circ H) + 2 b \sqrt{\frac{1 - \delta}{2N}}.
	\end{align}
\end{theorem}

To obtain Theorem 5, we need a proposition originally proved by Liu and Tao (see \cite{liu2016classification}). Here, we provide a more general proposition slightly extended from the original version by Liu and Tao. This proposition guarantees the generalization of our algorithm.

\begin{proposition}
	Suppose the conditional probability of observed noisy labels given features is $P(\tilde Y |X)$. Assume $\beta(X, \tilde{Y})l(f(X), \tilde Y)$ is controlled by an upper bound $b$. For any real-valued constant $\delta \in (0, 1)$, with probability at least $1 - \delta$, we have
	\begin{align}
		& \sup_{h \in H} \left|R(h, D, l) - \hat{R}(h, S, \beta l) \right| \nonumber \\
        \le & \max_{(X, \tilde{Y})} \beta(X, \tilde Y) \mathfrak{R}(l \circ H) + b \sqrt{\frac{1 - \delta}{2n}}.
	\end{align}
\end{proposition}

Theorem 5 is straightforward from the Proposition 1 and Lemma 3.

\section{Empirical Results}

We conduct experiments on generated synthetic datasets, UCI benchmark datasets \footnote{The UCI benchmark datasets can be obtained from \url{http://theoval.cmp.uea.ac.uk/matlab/}.}, and a real-world dataset TCDB \footnote{The dataset TCDB can be obtained from \url{http://cseweb.ucsd.edu/~elkan/posonly/}.}.

In all experiments, we estimate the boundary $l$ by two methods. The first one is to calculate the mean of $n'$ smallest $\Delta \tilde{P}(x)$; the second way is to find the $l$ by cross-validation. The value of $n’$ will be specified later. These $2$ methods are respectively denoted as PGPU and PGPUcv. For cross-validation, we divide the training sets into $5$ folds, and search for boundary $l$ from $-0.9$ to $-0.6$ with the step of $0.01$.

We compare our algorithm's performance with five baselines: SVM on PU data (denoted as SVM), \cite{elkan2008learning}'s method (denoted as Elkan), \cite{natarajan2013learning}'s second method (C-SVM) (denoted as Natarajan), \cite{liu2016classification}'s method (denoted as Liu), and SVM on clean data (denoted as clean). The results prove the feasibility of our method.

\begin{figure*}[h]
\centering
\subfigure[Clean Data]{\includegraphics[width=1.85in]{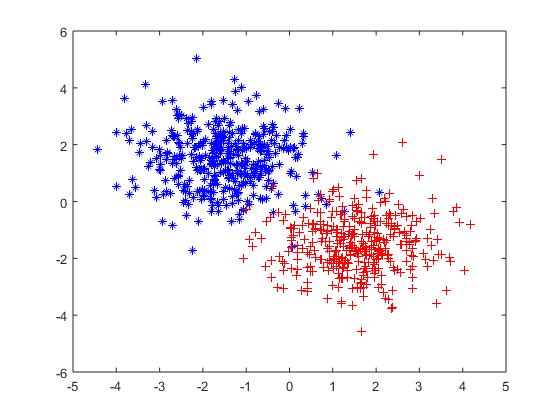}%
\label{fig_first_case}}
\hfil
\subfigure[PU Data]{\includegraphics[width=1.85in]{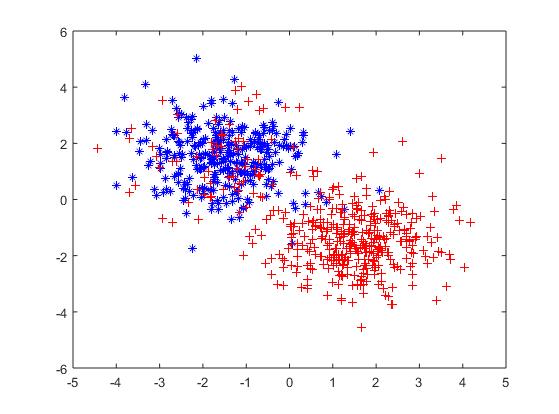}%
\label{fig_second_case}}
\hfil
\subfigure[Labelled Data]{\includegraphics[width=1.85in]{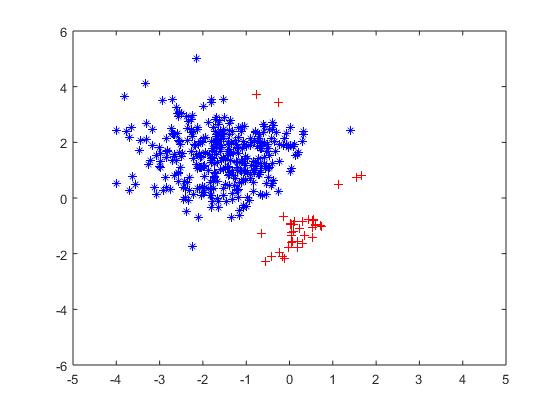}%
\label{fig_second_case}}
\caption{Illustrations of clean data, PU data, and labelled data via Bayesian optimal relabelling process. Red points are positive while blue ones are negative.}
\label{fig_sim}
\end{figure*}

\subsection{Simulations on Synthetic Data}

To start with, we generate 2-dimensional non-overlapping binary-class datasets to evaluate our algorithm. Positive and negative examples are sampled uniformly from $2$ triangles whose vertices are respectively $\{(-1, -1), (-1, 1), (1, 1)\}$ and $\{(-1, -1), (1, 1), (1, -1)\}$. There are $1000$ positive points and $1000$ negative points in the datasets. Then the conditional probabilities $P(Y|X)$ are estimated by \cite{platt1999probabilistic}'s method. Probabilistic gap $\Delta P(X)$ is further obtained. To generate the PU data, we then randomly flip positive labels to negative via $17$ settings of mislabelled rate:
\begin{itemize}
\item $9$ inverse settings:
\begin{gather}
	\rho(X) = \frac{\alpha}{(\alpha + \Delta P(X)(1 + \beta)}, \nonumber \\
 	\alpha = 0.1, 0.2, 0.3, \text{ } \beta = 0.5, 1.0, 1.5.
\end{gather}
    
\item $5$ linear settings:
\begin{gather}
	\rho(X) = \alpha (1 - \Delta P(X)), \nonumber \\
    \alpha = 0.2, 0.4, 0.6, 0.8, 1.0.
\end{gather}
    
\item $3$ constant settings:
\begin{gather}
	\rho(X) = 0.1, 0.2, 0.3.
\end{gather}
\end{itemize}

We further run all algorithm on the datasets. Each dataset was randomly split $10$ times, $75\%$ for training and $25\%$ for test. To estimate the boundary $l$, we use the mean of $n' = 3$ smallest $\Delta \tilde{P}(x)$. The results are as Table 1-3.

\begin{table*}[h!]
	\caption{Mean and Standard Deviation of Classification Accuracies of all Methods on Non-overlapping Generated Synthetic Dataset with Inverse mislabelled rate.}
	\label{sample-table}
    \tiny
	\centering
    \vbox{}
	\begin{tabular}{llllllll}
		\hline
		$(\alpha, \beta)$& SVM			& Elkan			& Natarajan		& Liu			& PGPU			& PGPUcv			&clean \\
		\hline
		$(0.1, 0.5)$	& $92.00\pm0.42$& $91.72\pm2.92$& $91.96\pm0.35$& $92.00\pm0.42$& $\bm{95.36}\pm\bm{1.48}$& $92.60\pm1.65$& $97.20$ \\
		$(0.1, 1.0)$	& $93.60\pm1.12$& $93.28\pm7.37$& $93.12\pm1.30$& $93.84\pm1.45$& $\bm{95.12}\pm\bm{1.43}$& $93.84\pm1.27$& $97.60$ \\
		\hline
		$(0.2, 0.5)$	& $92.64\pm0.78$& $91.44\pm5.02$& $91.84\pm0.89$& $93.24\pm0.84$& $\bm{94.44}\pm\bm{1.45}$& $92.88\pm0.92$& $99.20$ \\
		$(0.2, 1.0)$	& $94.28\pm0.80$& $94.48\pm1.01$& $93.04\pm0.66$& $94.60\pm0.63$& $\bm{96.64}\pm\bm{1.83}$& $94.64\pm1.04$& $98.40$ \\
		\hline
		$(0.3, 0.5)$	& $90.36\pm1.07$& $90.08\pm8.18$& $89.76\pm0.85$& $90.04\pm0.97$& $\bm{94.08}\pm\bm{2.62}$& $91.04\pm1.68$& $99.60$ \\
		$(0.3, 1.0)$	& $91.36\pm0.76$& $91.20\pm5.67$& $90.04\pm0.95$& $91.88\pm0.84$& $\bm{92.68}\pm\bm{2.00}$& $91.96\pm2.36$& $98.40$ \\
		\hline
	\end{tabular}
\end{table*}

\begin{table*}[h!]
	\caption{Mean and Standard Deviation of Classification Accuracies of all Methods on Non-overlapping Generated Synthetic Dataset with Linear mislabelled rate.}
	\label{sample-table}
    \tiny
	\centering
    \vbox{}
	\begin{tabular}{llllllll}
		\hline
		$\alpha$& SVM			& Elkan			& Natarajan		& Liu			& PGPU			& PGPUcv			&clean \\
		\hline
		$0.2$	& $95.84\pm0.84$& $96.40\pm0.82$& $95.92\pm1.05$& $96.32\pm1.05$& $\bm{97.44}\pm\bm{1.27}$& $96.16\pm1.10$& $98.80$ \\
		$0.4$	& $93.20\pm0.60$& $93.88\pm0.57$& $93.56\pm0.58$& $94.00\pm0.53$& $\bm{94.72}\pm\bm{4.86}$& $93.36\pm0.69$& $98.40$ \\
		$0.6$	& $91.40\pm1.58$& $91.52\pm1.03$& $90.96\pm1.15$& $91.32\pm1.31$& $91.24\pm0.95$& $\bm{91.64}\pm\bm{1.84}$& $99.20$ \\
		$0.8$	& $91.44\pm0.39$& $91.88\pm0.63$& $91.40\pm0.28$& $91.80\pm0.74$& $\bm{93.48}\pm\bm{1.70}$& $91.80\pm1.12$& $98.80$ \\
		$1.0$	& $91.40\pm1.58$& $91.44\pm0.97$& $90.96\pm1.15$& $91.48\pm1.35$& $\bm{91.72}\pm\bm{1.31}$& $91.56\pm1.72$& $99.20$ \\
		\hline
	\end{tabular}
\end{table*}

\begin{table*}[h!]
	\caption{Mean and Standard Deviation of Classification Accuracies of all Methods on Non-overlapping Generated Synthetic Dataset with Constant mislabelled rate.}
	\label{sample-table}
    \tiny
	\centering
    \vbox{}
	\begin{tabular}{llllllll}
		\hline
		$\alpha$& SVM			& Elkan			& Natarajan		& Liu			& PGPU			& PGPUcv			&clean \\
		\hline
		$0.1$	& $93.88\pm0.42$& $94.28\pm6.40$& $93.48\pm0.35$& $\bm{94.36}\pm\bm{0.57}$& $93.56\pm1.51$& $94.28\pm1.40$& $97.20$ \\
		$0.2$	& $94.52\pm0.60$& $\bm{97.84}\pm\bm{0.74}$& $94.16\pm0.74$& $96.92\pm1.02$& $95.27\pm1.67$& $94.84\pm1.01$& $98.80$ \\
		$0.3$	& $92.12\pm0.27$& $92.56\pm0.57$& $91.92\pm0.37$& $92.24\pm0.39$& $\bm{93.04}\pm\bm{0.95}$& $92.70\pm1.99$& $98.80$ \\
		\hline
	\end{tabular}
\end{table*}

Then we evaluate our algorithms on overlapping synthetic datasets. We generated the datasets in 2 steps. First, we uniformly sample $2000$ examples in the square with the vertices $(-1, -1), (-1, 1), (1, 1), (1, -1)$. Here, we use $(X^1, X^2)$ to denote the instances. Then, we label the examples as positive by the probability of $\max\{0, 0.5 - 10(X^1 - X^2)\}$. The percentage of examples that lie in the overlapping area $\ge 2.5\%$. Finally, we run all algorithm on the datasets. Each dataset was randomly split $10$ times, $75\%$ for training and $25\%$ for test. We still use the mean of $n' = 3$ smallest $\Delta \tilde{P}(x)$ to estimate the boundary $l$. The quantitative results are as as Table 4-6.

\begin{table*}[h]
	\caption{Mean and Standard Deviation of Classification Accuracies of all Methods on Overlapping Generated Synthetic Dataset with Inverse mislabelled rate.}
	\label{sample-table}
    \tiny
	\centering
    \vbox{}
	\begin{tabular}{lllllllll}
		\hline
		$\alpha$& $\beta$& SVM			& Elkan			& Natarajan		& Liu			& PGPU			& PGPUcv			&clean \\
		\hline
		$0.1$	& $0.5$	& $95.86\pm0.90$& $95.62\pm0.61$& $95.52\pm0.78$& $95.90\pm0.90$& $\bm{97.16}\pm\bm{0.36}$& $95.86\pm0.90$& $98.00$ \\
		$0.1$	& $1.0$	& $95.44\pm0.23$& $95.08\pm0.73$& $95.18\pm0.41$& $95.38\pm0.20$& $\bm{95.92}\pm\bm{0.48}$& $95.60\pm0.54$& $97.20$ \\
		\hline
		$0.2$	& $0.5$	& $95.70\pm0.60$& $94.44\pm3.48$& $95.52\pm0.53$& $95.82\pm0.48$& $\bm{96.04}\pm\bm{0.74}$& $95.60\pm0.58$& $97.00$ \\
		$0.2$	& $1.0$	& $93.58\pm0.52$& $94.22\pm1.20$& $92.90\pm0.34$& $93.74\pm0.42$& $\bm{95.68}\pm\bm{0.84}$& $93.98\pm1.09$& $97.60$ \\
		\hline
		$0.3$	& $0.5$	& $91.86\pm0.84$& $92.74\pm4.25$& $91.94\pm0.88$& $93.40\pm1.05$& $\bm{94.42}\pm\bm{0.49}$& $92.02\pm0.94$& $98.40$ \\
		$0.3$	& $1.0$	& $92.12\pm0.30$& $93.16\pm0.89$& $92.12\pm0.33$& $92.78\pm0.48$& $\bm{94.22}\pm\bm{1.09}$& $92.32\pm0.84$& $97.40$ \\
		\hline
	\end{tabular}
\end{table*}

\begin{table*}[h]
	\caption{Mean and Standard Deviation of Classification Accuracies of all Methods on Overlapping Generated Synthetic Dataset with Linear mislabelled rate.}
	\label{sample-table}
    \tiny
	\centering
    \vbox{}
	\begin{tabular}{llllllll}
		\hline
		$\alpha$& SVM			& Elkan			& Natarajan		& Liu			& PGPU			& PGPUcv			&clean \\
		\hline
		$0.2$	& $97.38\pm0.48$& $97.86\pm0.31$& $97.50\pm0.45$& $97.32\pm0.52$& $\bm{97.98}\pm\bm{0.53}$& $97.42\pm0.57$& $98.40$ \\
		$0.4$	& $95.88\pm0.69$& $96.52\pm0.56$& $95.18\pm0.86$& $96.04\pm0.74$& $\bm{97.34}\pm\bm{0.35}$& $96.12\pm0.95$& $98.00$ \\
		$0.6$	& $92.66\pm0.74$& $93.24\pm0.36$& $92.76\pm0.57$& $92.66\pm0.60$& $\bm{94.28}\pm\bm{0.77}$& $92.80\pm0.95$& $96.60$ \\
		$0.8$	& $91.26\pm0.42$& $91.84\pm0.66$& $91.02\pm0.48$& $91.32\pm0.40$& $\bm{92.60}\pm\bm{0.54}$& $91.42\pm0.58$& $97.60$ \\
		$1.0$	& $88.94\pm0.87$& $89.78\pm0.92$& $88.68\pm0.77$& $89.18\pm0.68$& $\bm{92.14}\pm\bm{0.94}$& $89.26\pm1.40$& $96.40$ \\
		\hline
	\end{tabular}
\end{table*}

\begin{table*}[h]
	\caption{Mean and Standard Deviation of Classification Accuracies of all Methods on Overlapping Generated Synthetic Dataset with Constant mislabelled rate.}
	\label{sample-table}
    \tiny
	\centering
    \vbox{}
	\begin{tabular}{llllllll}
		\hline
		$\alpha$& SVM			& Elkan			& Natarajan		& Liu			& PGPU			& PGPUcv			&clean \\
		\hline
		$0.1$	& $96.30\pm0.24$& $94.36\pm5.80$& $96.30\pm0.22$& $96.42\pm0.29$& $\bm{96.76}\pm\bm{0.53}$& $96.30\pm0.24$& $97.40$ \\
		$0.2$	& $94.64\pm0.63$& $93.70\pm10.30$& $94.30\pm0.71$& $\bm{96.80}\pm\bm{0.55}$& $95.58\pm0.80$& $95.00\pm1.23$& $98.00$ \\
		$0.3$	& $89.46\pm0.61$& $\bm{96.08}\pm\bm{0.33}$& $89.10\pm0.80$& $93.58\pm1.62$& $89.42\pm1.06$& $89.94\pm1.80$& $96.40$ \\
		\hline
	\end{tabular}
\end{table*}

\subsection{UCI Benchmark Dataset}

We also evaluate our methods on generated datasets based on UCI benchmarks, i.e., splice, banana, twonorm, image, and Heart. They respectively have $2,991$, $5,300$, $7.400$, $2,086$, and $270$ examples. We flip positive labels to negative via $9$ settings of mislabelled rate:
\begin{itemize}
\item $9$ inverse settings:
\begin{gather}
	\rho(X) = \frac{\alpha}{(\alpha + \Delta P(X)(1 + \beta)}, \nonumber \\
 	\alpha = 0.1, 0.2, 0.3, \text{ } \beta = 0.5, 1.0, 1.5.
\end{gather}
\end{itemize}.

We further run all algorithm on the datasets. Each dataset was randomly split $10$ times, $75\%$ for training and $25\%$ for test. Here, we use the minimization of $\Delta \tilde{P}(x)$ to estimate the boundary $l$. The quantitative results are shown in Table 7. As the size of the Heart dataset is too small (there are only $270$ instances), the cross-validation is not suitable. So we do not conduct experiments for the method PGPUcv for the Heart dataset.

We can see that the proposed PGPUcv mostly outperforms the baselines, empirically showing the advantages of the proposed method. PGPU works worse than PGPUcv may because the data in the UCI benchmarks is noisy or of small size, making it difficulty to estimate the conditional probabilities. 

\begin{table*}[h]
\caption{Mean and Standard Deviation of Classification Accuracies of all Methods on UCI Dataset with Inverse mislabelled rate (Whitenned).}
\label{sample-table}
\tiny
\centering
\vbox{}
\begin{tabular}{lllllllll}
	\hline
	Dataset	& $(\alpha, \beta)$			& SVM			& Elkan			& Natarajan		& Liu			& PGPU			& PGPUcv			&clean \\
	\hline
	Splice	& $(0.1, 0.5)$	& $56.91\pm1.69$& $55.09\pm7.10$& $56.53\pm1.66$& $58.11\pm2.35$& $48.71\pm3.28$& $\bm{59.15}\pm\bm{1.18}$	& $66.98$ \\
	Splice	& $(0.1, 1.0)$	& $56.62\pm1.39$& $53.11\pm6.11$& $56.00\pm1.31$& $57.65\pm1.65$& $50.31\pm6.31$& $\bm{58.80}\pm\bm{1.76}$	& $67.89$ \\
	Splice	& $(0.1, 1.5)$	& $56.61\pm0.97$& $50.64\pm7.88$& $55.92\pm1.22$& $57.80\pm1.26$& $51.36\pm4.72$& $\bm{57.84}\pm\bm{4.13}$	& $67.49$ \\
	Splice	& $(0.2, 0.5)$	& $56.72\pm1.22$& $50.40\pm7.38$& $56.23\pm1.43$& $57.49\pm2.18$& $50.08\pm5.40$& $\bm{59.46}\pm\bm{2.08}$	& $66.87$ \\
	Splice	& $(0.2, 1.0)$	& $56.03\pm1.46$& $47.28\pm5.21$& $55.04\pm0.73$& $\bm{57.32}\pm\bm{3.92}$& $51.55\pm7.17$& $57.15\pm3.20$	& $67.05$ \\
	Splice	& $(0.2, 1.5)$	& $55.19\pm0.65$& $47.61\pm7.30$& $54.77\pm0.71$& $54.42\pm1.08$& $52.95\pm5.17$& $\bm{59.20}\pm\bm{2.42}$	& $66.09$ \\
	Splice	& $(0.3, 0.5)$	& $55.74\pm1.38$& $52.19\pm12.1$& $55.41\pm1.45$& $56.33\pm7.77$& $52.62\pm6.12$& $\bm{56.98}\pm\bm{1.69}$	& $67.32$ \\
	Splice	& $(0.3, 1.0)$	& $77.17\pm2.78$& $69.80\pm2.33$& $75.69\pm3.11$& $79.91\pm1.67$& $80.07\pm2.41$& $\bm{80.21}\pm\bm{0.87}$	& $95.79$ \\
	Splice	& $(0.3, 1.5)$	& $81.44\pm2.18$& $68.77\pm3.23$& $80.07\pm2.32$& $78.42\pm0.81$& $81.48\pm3.07$& $\bm{83.64}\pm\bm{0.87}$	& $95.79$ \\
	\hline
	Banana	& $(0.1, 0.5)$	& $57.51\pm1.26$& $55.34\pm6.30$& $57.15\pm1.50$& $\bm{58.16}\pm\bm{1.36}$& $49.80\pm3.73$& $58.05\pm1.91$	& $68.71$ \\
	Banana	& $(0.1, 1.0)$	& $55.95\pm1.48$& $47.11\pm6.15$& $55.37\pm1.32$& $\bm{57.07}\pm\bm{1.79}$& $54.42\pm6.49$& $57.06\pm4.15$	& $66.83$ \\
	Banana	& $(0.1, 1.5)$	& $54.91\pm1.24$& $52.97\pm5.86$& $54.47\pm1.00$& $\bm{56.23}\pm\bm{1.53}$& $51.80\pm7.97$& $55.92\pm2.63$	& $66.05$ \\
	Banana	& $(0.2, 0.5)$	& $56.72\pm1.38$& $51.55\pm6.03$& $56.21\pm1.43$& $57.19\pm2.05$& $48.39\pm5.62$& $\bm{58.50}\pm\bm{1.45}$	& $67.73$ \\
	Banana	& $(0.2, 1.0)$	& $56.25\pm2.57$& $48.04\pm6.61$& $55.43\pm2.60$& $56.70\pm3.25$& $49.76\pm5.41$& $\bm{57.38}\pm\bm{4.59}$	& $65.83$ \\
	Banana	& $(0.2, 1.5)$	& $56.10\pm3.14$& $47.05\pm6.32$& $55.84\pm3.12$& $52.78\pm4.81$& $51.56\pm4.42$& $\bm{58.26}\pm\bm{4.78}$	& $66.45$ \\
	Banana	& $(0.3, 0.5)$	& $54.91\pm1.40$& $47.34\pm6.06$& $54.50\pm1.10$& $54.20\pm2.59$& $54.32\pm4.18$& $\bm{57.66}\pm\bm{2.56}$	& $66.06$ \\
	Banana	& $(0.3, 1.0)$	& $55.33\pm1.49$& $\bm{62.31}\pm\bm{14.5}$& $55.21\pm1.32$& $60.22\pm9.09$& $55.25\pm1.64$& $56.13\pm3.12$	& $66.12$ \\
	\hline
	Twonorm	& $(0.1, 0.5)$	& $97.52\pm0.35$& $97.47\pm0.51$& $97.48\pm0.33$& $97.64\pm0.36$& $96.96\pm0.50$& $\bm{97.65}\pm\bm{0.36}$	& $97.92$ \\
	Twonorm	& $(0.1, 1.0)$	& $97.42\pm0.23$& $97.53\pm0.47$& $97.43\pm0.23$& $97.48\pm0.27$& $96.99\pm0.25$& $\bm{97.56}\pm\bm{0.25}$	& $97.88$ \\
	Twonorm	& $(0.1, 1.5)$	& $97.54\pm0.27$& $96.88\pm1.10$& $97.55\pm0.27$& $97.56\pm0.25$& $96.86\pm0.34$& $\bm{97.59}\pm\bm{0.30}$	& $97.72$ \\
	Twonorm	& $(0.2, 0.5)$	& $96.51\pm0.46$& $96.23\pm2.34$& $96.41\pm0.49$& $96.70\pm0.39$& $\bm{97.15}\pm\bm{0.37}$& $97.11\pm0.42$	& $97.85$ \\
	Twonorm	& $(0.2, 1.0)$	& $96.91\pm0.38$& $96.26\pm1.46$& $96.85\pm0.37$& $97.10\pm0.34$& $96.97\pm0.49$& $\bm{97.24}\pm\bm{0.27}$	& $97.75$ \\
	Twonorm	& $(0.2, 1.5)$	& $97.23\pm0.27$& $96.86\pm1.43$& $97.22\pm0.31$& $\bm{97.28}\pm\bm{0.27}$& $97.02\pm0.39$& $\bm{97.28}\pm\bm{0.38}$	& $97.83$ \\
	Twonorm	& $(0.3, 0.5)$	& $94.43\pm0.51$& $95.40\pm2.05$& $94.39\pm0.53$& $95.77\pm0.45$& $\bm{97.21}\pm\bm{0.37}$& $96.05\pm0.79$	& $97.92$ \\
	Twonorm	& $(0.3, 1.0)$	& $96.16\pm0.64$& $96.98\pm1.30$& $96.20\pm0.63$& $96.61\pm0.59$& $\bm{97.16}\pm\bm{0.25}$& $96.85\pm0.61$	& $97.91$ \\
	Twonorm	& $(0.3, 1.5)$	& $96.59\pm0.36$& $96.74\pm1.09$& $96.54\pm0.46$& $96.86\pm0.33$& $96.86\pm0.33$& $\bm{97.11}\pm\bm{0.41}$	& $97.86$ \\
	\hline
	Image	& $(0.1, 0.5)$	& $75.57\pm5.96$& $63.07\pm6.15$& $74.33\pm5.32$& $67.57\pm5.14$& $68.45\pm6.82$& $\bm{76.00}\pm\bm{6.66}$	& $88.31$ \\
	Image	& $(0.1, 1.0)$	& $79.56\pm6.98$& $68.60\pm9.99$& $78.14\pm6.46$& $64.92\pm5.94$& $73.26\pm6.72$& $\bm{78.43}\pm\bm{7.01}$	& $88.33$ \\
	Image	& $(0.1, 1.5)$	& $\bm{80.42}\pm\bm{6.99}$& $69.08\pm6.93$& $78.89\pm6.68$& $67.84\pm1.65$& $72.49\pm8.30$& $77.95\pm7.61$	& $88.24$ \\
	Image	& $(0.2, 0.5)$	& $75.54\pm5.84$& $64.44\pm6.57$& $73.87\pm5.24$& $66.23\pm9.23$& $69.81\pm8.31$& $\bm{76.72}\pm\bm{5.72}$	& $87.70$ \\
	Image	& $(0.2, 1.0)$	& $71.92\pm4.03$& $68.47\pm4.90$& $\bm{70.90}\pm\bm{2.84}$& $69.12\pm2.03$& $64.44\pm3.97$& $70.82\pm2.17$	& $88.03$ \\
	Image	& $(0.2, 1.5)$	& $\bm{73.30}\pm\bm{6.81}$& $67.30\pm5.87$& $72.55\pm6.03$& $72.15\pm6.46$& $68.35\pm5.87$& $71.46\pm6.76$	& $88.45$ \\
	Image	& $(0.3, 0.5)$	& $64.75\pm4.55$& $65.90\pm4.02$& $63.26\pm4.89$& $\bm{68.08}\pm\bm{3.31}$& $65.24\pm3.58$& $67.49\pm2.81$	& $87.57$ \\
	Image	& $(0.3, 1.0)$	& $71.74\pm5.67$& $68.16\pm8.01$& $70.79\pm4.46$& $71.53\pm5.95$& $65.98\pm6.20$& $\bm{71.84}\pm\bm{4.91}$	& $88.58$ \\
	Image	& $(0.3, 1.5)$	& $72.72\pm2.73$& $66.17\pm6.15$& $71.69\pm2.60$& $70.42\pm2.43$& $67.57\pm1.94$& $\bm{72.76}\pm\bm{2.55}$	& $87.99$ \\
	\hline
	Heart	& $(0.1, 0.5)$	& $79.90\pm0.35$& $75.98\pm0.51$& $79.41\pm0.33$& $\bm{79.90}\pm\bm{0.36}$& $77.94\pm0.50$& $\text{NA}$	& $81.37$ \\
	Heart	& $(0.1, 1.0)$	& $\bm{83.82}\pm\bm{0.23}$& $81.86\pm0.47$& $83.33\pm0.23$& $81.37\pm0.27$& $76.96\pm0.25$& $\text{NA}$	& $86.27$ \\
	Heart	& $(0.1, 1.5)$	& $83.33\pm0.27$& $78.43\pm1.10$& $83.33\pm0.27$& $\bm{83.82}\pm\bm{0.25}$& $80.88\pm0.34$& $\text{NA}$	& $82.35$ \\
	Heart	& $(0.2, 0.5)$	& $78.92\pm0.46$& $76.96\pm2.34$& $79.90\pm0.49$& $\bm{80.88}\pm\bm{0.39}$& $78.92\pm0.37$& $\text{NA}$	& $81.86$ \\
	Heart	& $(0.2, 1.0)$	& $76.47\pm0.38$& $75.49\pm1.46$& $76.96\pm0.37$& $\bm{79.41}\pm\bm{0.34}$& $78.92\pm0.49$& $\text{NA}$	& $82.84$ \\
	Heart	& $(0.2, 1.5)$	& $76.96\pm0.27$& $\bm{82.35}\pm\bm{1.43}$& $77.45\pm0.31$& $80.88\pm0.27$& $76.92\pm0.39$& $\text{NA}$	& $80.88$ \\
	Heart	& $(0.3, 0.5)$	& $70.59\pm0.51$& $81.37\pm2.05$& $64.71\pm0.53$& $\bm{81.86}\pm\bm{0.45}$& $69.12\pm0.37$& $\text{NA}$	& $83.33$ \\
	Heart	& $(0.3, 1.0)$	& $77.45\pm0.64$& $\bm{83.82}\pm\bm{1.30}$& $77.94\pm0.63$& $83.33\pm0.59$& $79.41\pm0.25$& $\text{NA}$	& $83.33$ \\
	Heart	& $(0.3, 1.5)$	& $78.43\pm0.36$& $\bm{80.88}\pm\bm{1.09}$& $78.43\pm0.46$& $79.41\pm0.33$& $\bm{80.88}\pm\bm{0.33}$& $\text{NA}$	& $82.84$ \\
	\hline
\end{tabular}
\end{table*}

\subsection{Experiment on Real-world Data}

Elkan and Noto release a real-world document dataset\footnote{This dataset can be obtained from the website http://www.cs.ucsd.edu/users/elkan/posonly.} \cite{elkan2008learning} based on a biological database SwissProt. The positive dataset $P$ contains 2453 examples which are obtained from a dataset TCDB \cite{saier2006tcdb}. Meanwhile, there are 4906 records in the unlabelled dataset $U$, which are randomly selected from SwissProt excluding those in the TCDB. In other words, those 2 datasets $P$ and $U$ are disjoint. Furthermore, Das et al. also manually label the unlabelled dataset \cite{das2007finding}. They identify 348 positive members from the whole dataset $N$, and name the new positive examples as $Q$ and the leftover is $N = U - Q$.

Here is an example in the TCDB. It is labeled as positive.
\newline
\begin{displayquote}
AC   Q57JA3; \\
p-hydroxybenzoic acid efflux pump subunit aaeA (pHBA efflux pump protein A). \\
"The genome sequence of Salmonella enterica serovar Choleraesuis, a highly invasive and resistant zoonotic pathogen."; \\
-!- SUBCELLULAR LOCATION: Cell inner membrane; single-pass membrane \\
    protein (Potential). \\
-!- MISCELLANEOUS: Both aaeA and aaeB are required for the function of \\
    the efflux system (By similarity). \\
-!- SIMILARITY: Belongs to the membrane fusion protein (MFP) \\
    (TC 8.A.1) family. \\
 GO:0006810; P:transport; IEA:HAMAP. \\
Complete proteome; Inner membrane; Membrane; Transmembrane; Transport. \\
TMONE \\
\end{displayquote}
An important challenge is that all examples in this real-world dataset are documents, and therefore, can not be processed directly. A series of work has presented on the topic of extracting vector representations for documents. In this paper, we employ Doc2Vec to represent all examples in our real-world dataset as vetors \cite{mikolov2013efficient, mikolov2013distributed, le2014distributed}. We represent each document as a $32$-dimensional real-value vector.

Comparison experiments are conducted on this real-world dataset. Empirical results, as shown in Table 7, are in agreement with our method.

\begin{table*}[h]
	\caption{Mean and Standard Deviation of Classification Accuracies of all Methods on Generated Synthetic Dataset with Linear mislabelled rate.}
	\label{sample-table}
    \smaller
	\centering
    \vbox{}
	\begin{tabular}{lllllll}
		\hline
		\multicolumn{7}{c}{Method} \\
		\hline
		SVM			& Elkan			& Natarajan		& Liu			& PGPU			& PGPUcv			&clean \\
		\hline
		$83.32$		& $84.14$		& $83.27$		& $80.66$		& $84.30$		& $\bm{84.36}$ 		& $83.98$\\
		\hline
	\end{tabular}
\end{table*}

In the experiments, our algorithms PGPU and PGPUcv outperform all others in most situations, though sometimes neither of them is the best. In our opinion, there are two reasons that make our algorithms not so good. Firstly, the Bayesian optimal relabelling process can not label all instances, which leads to a smaller labelled sample for training a classifier on. Therefore, when the original sample size is small, our algorithms do not out perform others. Secondly, our algorithms rely on an assumption that the mislabelled rate is monotone decreasing with the probabilistic gap $\Delta P(X)$. Therefore, when this assumption does not hold (for example, when the mislabelled rate is a constant), our algorithms could have a bad performance.

In addition, the performance of Elkan and Noto’s method and our algorithms PGPU and PGPUcv on PU data sometimes overcomes the one of SVM on the clean data. In our opinion, this phenomenon is because of the reweighting techniques used in these methods. Reweighting techniques can improve the performance of the based classification methods in many situations \cite{yang2007weighted}. When the mislabelled rate is not too high, it is possible for reweighting methods on PU data to outperform the SVM on the clean data.

\section{Conclusion and Future Work}

Learning classifier from positive and unlabelled data has many real-world applications. To solve this problem, in this paper, we focus on the measure of the difficulty of labelling examples, and develop a model based on an innovative conception, probabilistic gap. With help of the probabilistic-gap model, we propose a novel relabelling algorithm. This method can provide labels to a group of unlabelled examples, which are identical to the ones assigned by the Bayesian optimal classifier. At the end of this paper, the empirical results of experiments illustrate our method's efficiency and support the theoretical analysis.

In the future work, we are interested in developing models more sophisticated and applicable for mislabelled rates. These models could reveal the nature of PU learning and help to boost classification algorithms' performance.

In addition, our proposed method does not use the examples in the sub-domain $\Delta P(X) \in (l,0)$. We call this phenomenon as {\it domain bias}. The domain bias could undermine the performance of the classification, which cannot be completely solved by the KMM technique. To address this issue, two methods are probably applicable:
\begin{enumerate}
\item {\it Active learning}. Using active learning methods\cite{tong2001active}, we can select a group of unlabelled instances in the sub-domain $\Delta P(X) \in (l,0)$ and hire human experts to label them.

\item {\it Semi-supervised learning}. In our method, after the Bayesian optimal relabelling process, all examples have their labels except a group that lie in the sub-domain $\Delta P(X) \in (l,0)$. In our methods, these data is discarded. Meanwhile, semi-supervised learning exactly focuses on the problem where the dataset contains both labelled and unlabelled data \cite{luo2018semi,akbarnejad2018efficient,yu2018semi}. Therefore, the methods of semi-supervised learning could be helpful.
\end{enumerate}

\subsection*{Acknowledgement}

This work was supported in part by Australian Research Council Projects FL-170100117, DP-180103424, IH-180100002, and DE-190101473. The authors sincerely thank Jessa Bekker and Jesse Davis for the helpful discussions.

{
\normalsize
\bibliographystyle{plain}
\bibliography{PU_learning}
}

\end{document}